\documentclass[a4paper,fleqn]{templates/els-cas-templates/cas-dc}

\usepackage[numbers,]{natbib}
\def\tsc#1{\csdef{#1}{\textsc{\lowercase{#1}}\xspace}}
\tsc{WGM}
\tsc{QE}
\usepackage{cleveref}
\usepackage{lipsum}
\usepackage{multicol}
\usepackage[OT1]{fontenc} 
\usepackage{bm}
\usepackage{textcomp}

\crefname{chapter}{Chap.}{Chap.}
\crefname{section}{Sec.}{Sec.}
\crefname{algorithm}{Alg.}{Alg.}
\crefname{table}{Tab.}{Tab.}
\crefname{figure}{Fig.}{Fig.}

\definecolor{airforceblue}{rgb}{0.36, 0.54, 0.66}
\definecolor{arsenic}{rgb}{0.23, 0.27, 0.29}
\hypersetup{
    colorlinks=true,
    linkcolor=arsenic,
    filecolor=magenta,      
    urlcolor=airforceblue,    
    pdfpagemode=FullScreen,
    }

\newtheorem{remark}{Remark}
\newtheorem{theorem}{Theorem}

\newtheorem{assumption}{Assumption}
\newtheorem{problem}{Problem}
\newtheorem{corollary}{Corollary}

\newcommand{\norm}[1]{\left \lVert #1 \right \rVert}

\newproof{proof}{Proof}
\begin{document}
\let\WriteBookmarks\relax
\def\floatpagepagefraction{1}
\def\textpagefraction{.001}

\shorttitle{TBOD: A Framework for Lightweight Sensor Fusion}    

\shortauthors{Oliva et al.}  

\title [mode = title]{Trajectory Based Observer Design: A Framework for Lightweight Sensor Fusion}  

\tnotemark[1]
\tnotetext[1]{This project has received funding from the \textit{Jack Buncher} and was submitted to Control Engineering Practice on the 27th of February 2025. We would like to thank the Biorobotics Lab at CMU for helpful discussions.}

\author[1]{Federico Oliva}[orcid=0000-0003-4694-6339]
\ead{federicoo@campus.technion.ac.il}

\author[1]{Tom Shaked}[orcid=0000-0002-1811-8731]
\ead{shakedtom@campus.technion.ac.il}

\author[2]{Daniele Carnevale}[orcid=0000-0001-6214-7938]
\ead{daniele.carnevale@uniroma2.it}

\author[2]{Amir Degani}[orcid=0000-0002-4813-8506]
\ead{adegani@technion.ac.il}
\cormark[2]
\cortext[2]{Corresponding author}

% \fntext[1]{Works at the \href{https://cear.net.technion.ac.il/}{CEAR Lab}.}

% \credit{This project has received funding from the Jack Buncher and was submitted to IEEE Sensors on the 6th of February 2025. We would like to thank the Biorobotics Lab at CMU for helpful discussions.}

\affiliation[1]{organization={Civil and Environmental Engineering and Technion Autonomous Systems Program, Technion - Israeli Institute of Technology},
            addressline={Technion City}, 
            city={Haifa},
%          citysep={}, % Uncomment if no comma needed between city and postcode
            postcode={3200003}, 
            country={Israel}}

\affiliation[2]{organization={Department of Civil Engineering and Computer Science, Tor Vergata university of Rome},
            addressline={Via del Politecnico, 1}, 
            city={Roma},
%          citysep={}, % Uncomment if no comma needed between city and postcode
            postcode={00133}, 
            country={Italy}}



\begin{abstract}
Efficient observer design and accurate sensor fusion are key in state estimation. This work proposes an optimization-based methodology, termed \textit{Trajectory Based Optimization Design} (TBOD), allowing the user to easily design observers for general nonlinear systems and multi-sensor setups. Starting from parametrized observer dynamics, the proposed method considers a finite set of pre-recorded measurement trajectories from the nominal plant and exploits them to tune the observer parameters through numerical optimization. This research hinges on the classic observer's theory and Moving Horizon Estimators methodology. Optimization is exploited to ease the observer's design, providing the user with a lightweight, general-purpose sensor fusion methodology. TBOD's main characteristics are the capability to handle general sensors efficiently and in a modular way and, most importantly, its straightforward tuning procedure. The TBOD's performance is tested on a terrestrial rover localization problem, combining IMU and ranging sensors provided by \textit{Ultra Wide Band} antennas, and validated through a motion-capture system. Comparison with an \textit{Extended Kalman Filter} is also provided, matching its position estimation accuracy and significantly improving in the orientation. 
\end{abstract}

\begin{graphicalabstract}
\includegraphics[width=\textwidth,trim={0 0 0 0},clip]{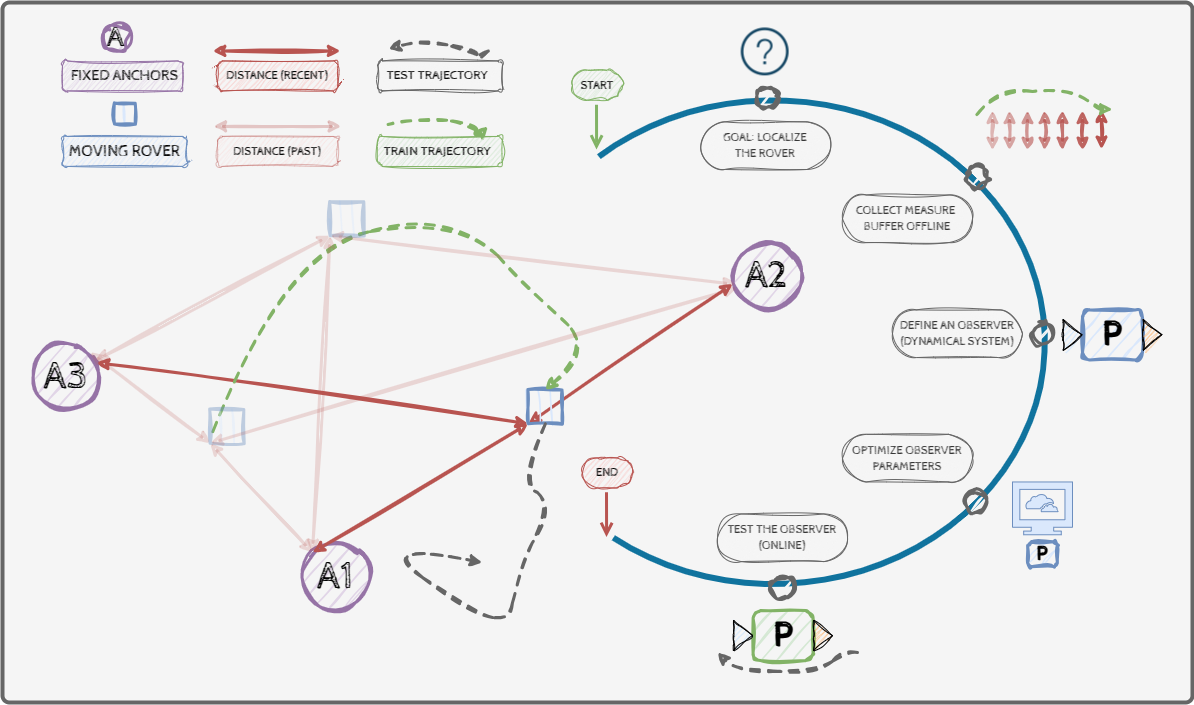}
\end{graphicalabstract}

\begin{highlights}
\item The paper proposes a methodology for designing a general-purpose observer for generic nonlinear systems (TBOD).
\item TBOD uses offline recorded trajectories of the nominal plan to tune the observer's parameters.
\item The paper tested the TBOD on a rover localization problem, provided stability analysis, and compared it with an EKF.
\item The TBOD matched the EKF's precision in position and improved orientation accuracy by an order of magnitude.
\end{highlights}

% Keywords

\maketitle

\setlength\parindent{0pt}

% INTRO
\section{Introduction}
\label{intro}

Observation Theory studies how to design dynamic systems fully reconstructing the state $\bm{\xi}$ of a generic plant $\mathcal{M}$, starting from a limited set of measured signals $\bm{y}$. Fundamental results solving the so-called Observation problem were obtained in the 60s, with the Kalman Filter (KF) and Luenberger observer formulations \cite{Kalman1960,Kalman1961,Luenberger1971}. 

\medskip
The Observation Problem has been constructively solved for linear dynamics, while only for particular classes of systems in the nonlinear context \cite{Bernard2022}. Specifically, KF is the statistically optimal solution for linear systems with Gaussian measurement and process noise \cite{Cox1964}. The KF extension through linearization, i.e., the Extended Kalman Filter (EKF), provides an effective tool for general nonlinear input-affine systems \cite{Keller1987,Besancon1997}. The EKF's properties have been exploited through several methods, such as output injection \cite{Krener1983,Bossane1989,Plestan1997}, linearization  \cite{Krener1985,Boutayeb1997,Glielmo1999,Guay2002,Djuric2003,Hightower2003,Julier2004,Daum2005,Barrau2017}, and differential observability \cite{Gauthier2001}. High-gain observers \cite{Tornambe1989,Deza1992,Khalil2014,Gismondi2022,Gismondi2023} are another relevant class of observers, characterized by an easy tuning process controlled by a single parameter. Observation Theory also started exploiting the growing computational power, introducing optimization in novel design methodologies, like Full Information Estimators (FIE) \cite{Rao2003} and Moving Horizon Estimators (MHE) \cite{Michalska1995,Kang2006,Rakovic2018}. For both of them, stability proofs are available \cite{Wynn2014,Schiller2021,Schiller2022,Schiller2023}. Even though some works move towards real-time solutions \cite{Kuhl2011,Oliva2022,Desai2023}, the computational burden still represents the main bottleneck.

\subsection{Novelty}
\label{intro:novelty}

This work proposes an optimization-based methodology, termed \textit{Trajectory Based Optimization Design} (TBOD), allowing the user to quickly design observers for general nonlinear systems and multi-sensor setups. Starting from a general and parametrized observer structure, a set of trajectories of the nominal plant is considered and exploited to tune the observer through numerical optimization. The main strengths of TBOD are:

\begin{itemize}
    \item \textbf{lightweight}: the entire computational burden is focused on the offline tuning procedure.
    \item \textbf{specificity}: the TBOD method allows fine-tuning observers for specific trajectories occurring in some applications (e.g., satellites orbiting, area patrolling).
    \item \textbf{modularity}: the TBOD method can be easily adapted and scaled up for different sensor setups.
\end{itemize}

\medskip
The basics of TBOD have been introduced in \cite{Oliva2023}, addressing a simple position estimate problem cast as a linear observer design on simulated data. The work also compares the TBOD to a KF and a Linear Matrix Inequality (LMI) design method. The current work generalizes the TBOD for generic nonlinear systems (i.e., a joint position-orientation observer in this case) and provides a deeper analysis with more structured theoretical stability guarantees. Furthermore, the accuracy of the proposed method is validated through experimental data, positioning the TBOD as a valid general-purpose sensor fusion approach.

\subsection{Performance evaluation}
\label{intro:performance}
In this work, the effectiveness of TBOD is shown in a real case study. Specifically, the TBOD method is exploited to design a position and orientation observer on a rover provided with an IMU sensor and ranging measurements from Ultra Wide Band (UWB) antennas\cite{Xia2022,Guler2023,Zhang2024}. Such a scenario was selected due to the rise of interest in navigation for patrolling and surveillance systems, which strongly rely on the agent's localization. For the observer structure, results from the classic Luenberger observer \cite{Luenberger1971} and hybrid systems \cite{Goebel2009} are exploited, as they well describe intermittent measurements \cite{Sferlazza2018,Alonge2019,Berkane2019}. Indeed, EKF represents the standard solution for localization problems \cite{Kaczmarek2022,Kim2020,Armesto2004,Mourikis2007,Armesto2008,Mehra1970,Mohamed1998,Barrau2017}, and it is therefore considered the benchmark for the TBOD performance evaluation, together with another widespread algorithm, i.e., the Particle Filter (PF) \cite{Djuric2003}. 

\subsection{Structure}
\label{intro:struct}
This work unfolds as follows: a general recap of the Observation Problem is presented in \cref{probstat}. The description of the TBOD methodology is provided in \cref{meth}. \Cref{expset} is devoted to showing the proposed approach's performance in the above localization problem. Results are described in \cref{res}, while conclusions are drawn in \cref{concl}.

% PROBLEM STATEMENT
\section{The Observation Problem}
\label{probstat}

Consider the following continuous-time nonlinear system

\vspace{-5pt}
\begin{equation}    
    \mathcal{M} \ : \
    \begin{cases}
        \dot{\bm{\xi}}  & = \bm{f}(\bm{\xi},\bm{u}) + \bm{w}  \\
        \bm{y}          & = \bm{h}(\bm{\xi},\bm{u}) + \bm{\nu}
    \end{cases}\quad ,
    \label{probstat:eqn:general_ss}
\end{equation}

\medskip
where $(\bm{\xi}, \bm{w}) \in \mathbb{R}^n$ are respectively the state and process noise vectors, $\bm{u} \in \mathbb{R}^m$ the control input, and $(\bm{y},\,\bm{\nu}) \in \mathbb{R}^p$ the output and measurement noise vectors. The $\bm{f}$ map is assumed to be sufficiently regular, and \eqref{probstat:eqn:general_ss} is referred to as the nominal plant $\mathcal{M}$, where the time dependence is neglected to ease the notation. Consider now the problem of reconstructing the state vector $\bm{\xi}$ from the measurements $\bm{y}$. The classic approach in observation theory consists of designing a dynamical system $\hat{\mathcal{M}}$, whose state $\hat{\bm{\xi}}$ will eventually converge to $\bm{\xi}$ starting from any initial condition $\hat{\bm{\xi}}_0$. Thus, the observer's dynamics are

\begin{equation}
    \hat{\mathcal{M}} \ : \
    \begin{cases}
        \dot{\hat{\bm{\xi}}}  & = \hat{\bm{f}}(\hat{\bm{\xi}},\bm{u},\bm{y})  \\
        \hat{\bm{y}}          & = \hat{\bm{h}}(\hat{\bm{\xi}},\bm{u})
    \end{cases} \quad ,
    \label{probstat:eqn:observer_ss}
\end{equation}

\medskip
where $\hat{\bm{\xi}} \in \mathbb{R}^n$ is the observer state vector, $\bm{u} \in \mathbb{R}^m$ and $\bm{y} \in \mathbb{R}^p$ are the control input and output measurements as in $\mathcal{M}$, respectively. Thus, the Observation Problem is defined:

\medskip
\begin{problem}[Observation Problem]
    \label{probstat:def:obs_problem}
        Consider the nominal plant $\mathcal{M}$ in  \eqref{probstat:eqn:general_ss} and the observer $\hat{\mathcal{M}}$ in  \eqref{probstat:eqn:observer_ss}. The Observation Problem consists of finding a suitable set of dynamics and output mapping $(\hat{\bm{f}},\hat{\bm{h}})$ such that the estimation error system $\bm{e} = \bm{\xi}-\hat{\bm{\xi}}$ has a globally uniformly asymptotic practically stable origin, i.e. for any initial condition of $\mathcal{M}$ and $\hat{\mathcal{M}}$, $\bm{e}$ remains bounded and the following holds:

    \begin{equation}
        \label{probstat:eqn:observer_convergence}
        \underset{t\rightarrow +\infty}{\lim} \norm{\bm{\xi} - \hat{\bm{\xi}}} = \underset{t\rightarrow +\infty}{\lim} \norm{\bm{e}} \leq \epsilon, \quad \epsilon \in \mathbb{R}_{\geq 0} \ , 
    \end{equation}
\end{problem}

% METHOD
\section{The TBOD method}
\label{meth}

\subsection{The observer's structure} 
\label{meth:TBOD}

Generally speaking, the observer dynamics can be split into two main terms \cite{Luenberger1971,Thrun2006}: 

\begin{itemize}
    \item \textbf{plant replica term:} describes the system evolution when the Observation Problem is solved. This term is referred to as $\hat{\bm{f}}_R(\hat{\bm{\xi}},\bm{u},\bm{\Gamma)}$, where $\bm{\Gamma}$ is a design parameters vector.
    \item \textbf{injection term:} describes the dynamics depending on the output mismatch $(\bm{y} - \hat{\bm{y}})$. It vanishes when the Observation Problem \ref{probstat:def:obs_problem} is solved. This term is referred to as $\hat{\bm{f}}_I(\hat{\bm{\xi}},\bm{y},\bm{\Delta)}$, and $\bm{\Delta}$ is a design parameters vector.
\end{itemize}

\medskip
The observer's output mapping is referred to as $\hat{\bm{h}}(\hat{\bm{\xi}},\bm{u})$ and depends on the sensor setup considered. Thus, the generic parametric equations for an observer are:

\begin{equation}
    \label{meth:eqn:parametricObserver}
    \hat{\mathcal{M}} \ : \ 
    \begin{cases}
        \dot{\hat{\bm{\xi}}} &= \hat{\bm{f}}_R(\hat{\bm{\xi}},\bm{u},\bm{\Gamma)} + \hat{\bm{f}}_I(\hat{\bm{\xi}},\bm{y},\bm{\Delta)} \\
        \hat{\bm{y}} &= \hat{\bm{h}}(\hat{\bm{\xi}},\bm{u})
    \end{cases} \quad .
\end{equation}

\subsection{The optimization problem} 
\label{meth:optprob}
Generally speaking, the design of an observer consists in structuring the dynamics in  \eqref{meth:eqn:parametricObserver}, and finding the best values for $(\bm{\Gamma},\bm{\Delta})$ solving Problem \ref{probstat:def:obs_problem}. The TBOD method simplifies the design procedure using an optimization-based numerical approach to tune the observer parameters.

\medskip
The TBOD method starts from a general structure of the observer consisting of the \textit{plant replica term} $\hat{\bm{f}}_R(\hat{\bm{\xi}},\bm{u},\bm{\Gamma})$ and \textit{injection term} $\hat{\bm{f}}_I(\hat{\bm{\xi}},\bm{y},\bm{\Delta})$, and solves the design through an optimization problem on $(\bm{\Gamma},\bm{\Delta})$. The TBOD method considers a set of $M$ different trajectories $(\bm{\xi}_j,\hat{\bm{\xi}}_j)$ for $(\mathcal{M},\hat{\mathcal{M}})$ within a time interval $[0, \ T]$, and $j\in\{1,\dots,M\}$. The related estimation error trajectories are defined as $\bm{e}_j \triangleq \bm{\xi}_j-\hat{\bm{\xi}}_j$, while the output mismatch for each trajectory is defined as $\bm{\delta}_{y,j} \triangleq \bm{y}_j-\hat{\bm{y}}_j$. From  \eqref{meth:eqn:parametricObserver}, both $\bm{e}_j$ and $\bm{\delta}_{y,j}$ depend on the parameters $(\bm{\Gamma},\bm{\Delta})$. Thus, the design process is formalized as an optimization problem
\begin{equation}
	(\overline{\bm{\Gamma}},\overline{\bm{\Delta}}) =
	\begin{aligned}
		 & \underset{(\bm{\Gamma},\bm{\Delta})}{\text{argmin}}
		 &                             & \sum\limits_{j=1}^M\int\limits_0^T\bm{\delta}_{y,j}(t)^T\text{W}\bm{\delta}_{y,j}(t)dt
	\end{aligned} \ ,
	\label{meth:eqn:TBOD_prob_def}
\end{equation}
where $\text{W}\in\mathbb{R}^{p\times p}$ is a scaling matrix whose role is to normalize all the output mismatch terms. 
The design takes into account measurement noise, as it relies on the actual error trajectories $\bm{e}_j$, affecting the output mismatch $(\bm{y}_{j}(t) - \hat{\bm{y}}_j(t))$. Clearly,  \eqref{meth:eqn:TBOD_prob_def} defines a continuous-time problem because it considers the integral of the output mismatch. However, in the actual implementation, $\norm{\bm{y}_{j}(t) - \hat{\bm{y}}_j(t)}$ is available at specific samples, depending on the sensors' sampling. Thus, \eqref{meth:eqn:TBOD_prob_def} is eventually implemented in discrete-time form, i.e.,

\small
\begin{equation}
	(\overline{\bm{\Gamma}},\overline{\bm{\Delta}}) =
	\begin{aligned}
		 & \underset{(\bm{\Gamma},\bm{\Delta})}{\text{argmin}}
		 &                             & \sum\limits_{j=1}^M\sum\limits_{p=1}^P\bm{\delta}_{y,j}(t_{p,j})^T\text{W}\bm{\delta}_{y,j}(t_{p,j})
	\end{aligned},
	\label{meth:eqn:TBOD_prob_def_dist}
\end{equation}
\normalsize

where the sampling frequency is defined for each trajectory as $F_{Hz,j} = T_{s_{p,j}}^{-1}$, with $T_{s_{p,j}} \triangleq t_{p+1,j} - t_{p,j}$.

\subsection{Comments}
\label{meth:comments}

The design of TBOD is lightweight and practical: it involves defining injection terms and assigning tunable gains, a process that is simple and does not require system-specific derivations. The hybrid observer structure used in the proposed case study already accounts for multi-rate sensor inputs, illustrating how the framework naturally accommodates real-world setups. Users are free to adopt more complex or tailored observer structures from the literature and simply tune the associated parameters. In the following two paragraphs, an additional clarification on the role of the training trajectories and the sensor setup used during optimization is provided.

\smallskip
\paragraph*{Selection of the training trajectories:}
A first comment regards the choice of reference trajectories, which impacts how well the observer parameters generalize to unseen scenarios. While this is sometimes related to concepts such as Persistent Excitation (PE) in system identification \cite{Stoica1993}, the goal here is not to identify the system but to synthesize observer parameters that minimize the estimation error over a representative set of trajectories. Rich trajectories help expose the estimation error dynamics across a wide range of operating conditions, enabling the optimization process to tune the parameters more effectively. Therefore, the selection of training trajectories is critical not for system identifiability, but for ensuring good observer performance across the desired deployment envelope. In this regard, two aspects stand out:

\begin{enumerate}    
    \item \textbf{robustness}: the TBOD becomes more robust as the number of trajectories increases. The set of $M$ trajectories considered in \eqref{meth:eqn:TBOD_prob_def_dist}-\eqref{meth:eqn:TBOD_prob_def_dist_groundtruth} can be seen as a \textit{training set}, which will be referred to as $\mathcal{T}$.
    \item \textbf{specificity}: the more specific the trajectories in $\mathcal{T}$, the more the observer's parameters will fit a particular scenario, e.g., patrolling, palletting, etc.
\end{enumerate}

\medskip
A tradeoff between robustness, specificity, and feasibility should be considered. As shown by the proposed case study in \cref{expset}, this gray area between specificity and general purpose allows the TBOD method to improve the estimation performance compared to more standard observers.

\paragraph*{Considerations on the training set's noise:}
In the TBOD optimization problem, the cost function depends on the system output $\bm{y}$ and its estimate $\hat{\bm{y}}$. In the observer's tuning process, the trajectories of the training set $\mathcal{T}$ are used to compute this cost and are assumed to be available offline.

\medskip
Usually, such a set of reference trajectories is measured in a controlled environment, such as a lab or a dedicated facility, where sensors measuring the system state with high precision might be available. For instance, motion capture systems, such as the \href{https://www.vicon.com/}{Vicon\texttrademark} \cite{Vicon}, can measure the position of a rigid body with precision in the order of millimeters. These tools are usually referred to as \textit{ground truth} sensors, and if available, they can provide the training set $\mathcal{T}$ of reference trajectories for the TBOD method. In this case, the optimization problem in  \eqref{meth:eqn:TBOD_prob_def_dist} can be substituted by

\begin{equation}
    (\overline{\bm{\Gamma}},\overline{\bm{\Delta}}) = 
        \begin{aligned}
            & \underset{(\bm{\Gamma},\bm{\Delta})}{\text{argmin}} \sum\limits_{j=1}^M\sum\limits_{p=1}^P\bm{e}_j(t_p)^T\text{W}\bm{e}_j(t_p).
        \end{aligned}
    \label{meth:eqn:TBOD_prob_def_dist_groundtruth}
\end{equation}

When such high-precision reference trajectories (i.e., ground truth) are available, they can be directly used to measure the true plant state. This allows replacing the output-mismatch-based cost (i.e., $\propto \bm{\delta}_{y,j}$) in \eqref{meth:eqn:TBOD_prob_def_dist} with the state error-based cost (i.e., $\propto \bm{e}_{j}$) in \eqref{meth:eqn:TBOD_prob_def_dist_groundtruth}, often yielding improved performance.

\medskip
However, the use of ground-truth data is not a requirement of the TBOD method but an enhancement. In systems where such measurements are expensive or unavailable, TBOD can still be applied by optimizing \eqref{meth:eqn:TBOD_prob_def_dist} to tune the observer, which only requires access to system outputs. While this may reduce performance due to sensor noise, it preserves the applicability of TBOD to a broader class of systems, including those where only onboard sensors are available.

\medskip
Importantly, once the optimal parameters $(\bm{\overline{\Gamma}}, \bm{\overline{\Delta}})$
are obtained offline by solving either \eqref{meth:eqn:TBOD_prob_def_dist} or \eqref{meth:eqn:TBOD_prob_def_dist_groundtruth}, the observer defined in equation \eqref{meth:eqn:parametricObserver} can be implemented online without requiring any further optimization. This ensures that the method remains practical for real-time applications, even when trained with limited or noisy data.

\medskip
To conclude, it should be further remarked that the TBOD method can straightforwardly adapt to very different sensor setups, with no need to change the structure of \eqref{meth:eqn:TBOD_prob_def_dist}-\eqref{meth:eqn:TBOD_prob_def_dist_groundtruth}. This is not the case, for instance, for standard KF observers, where the output mapping directly affects the injection term. 

% SETUP AND STABILITY
\section{Case study}
\label{expset}

As a case study, the TBOD is here used to design a position and orientation observer on a rover. The rover's sensor setup comprises an IMU and three UWB antennas providing range measurements to four fixed anchors. The IMU sensor provides acceleration and angular velocity measurements. The three UWB tags, sensing ranges to four fixed anchors with known locations, are used to compute the position and orientation of the rover. Indeed, three tags have been deployed because a rigid body in three-dimensional space can be fully localized if at least three non-collinear points on the body are known \cite{Siciliano2016}.  The setup also includes six Vicon T10S cameras that track the rover location with an order of accuracy of millimeters. Such a ground truth system should not be considered part of the sensor's setup, but only as a validation for the experiment's accuracy.

\subsection{Observer design}
\label{expset:design}
The nominal plant dynamics are modeled as follows

\small
\begin{subequations}
    \begin{align}
	&\mathcal{P} \ : \
	\begin{cases}
		\dot{\bm{p}} & = \bm{v}          \\
		\dot{\bm{v}} & = \mathcal{R}_{\mathcal{L}}^{\mathcal{G}}\bm{u}_a = \bm{a}\\
		\dot{\bm{b}} & = 0 \\
        \dot{\bm{q}} & = \dfrac{1}{2}\bm{\Omega}(\bm{\omega})\bm{q}       \\
		\dot{\bm{\omega}} & = \mathcal{R}_{\mathcal{L}}^{\mathcal{G}}\bm{u}_{\omega}\\
	\end{cases}
	\rightarrow
    \begin{cases}
		\dot{\bm{\xi}}_p & = \quad
        \begin{bmatrix}
            \bm{v} \\            \mathcal{R}_{\mathcal{L}}^{\mathcal{G}}\bm{u}_a \\
            0
        \end{bmatrix}\\
        \dot{\bm{q}} & = \dfrac{1}{2}\bm{\Omega}(\bm{\omega})\bm{q}       \\
		\dot{\bm{\omega}} & = \mathcal{R}_{\mathcal{L}}^{\mathcal{G}}\bm{u}_{\omega}\\
	\end{cases} \quad ,
	 \label{expset:design:eqn:plant_dyn}\\ \notag \\
 	& \hspace{30pt} \bm{y} =
	\begin{bmatrix}
		\bm{y}_d \\
		\bm{y}_a \\
        \bm{y}_\omega 
	\end{bmatrix}
	=
	\begin{bmatrix}
		\bm{d}^{\star} \\
		\bm{u}_a \\
        \bm{\omega}
	\end{bmatrix}
	+
	\begin{bmatrix}
		\bm{\nu}_d \\
		\bm{b} + \bm{\nu}_a \\
        \bm{\nu}_\omega
	\end{bmatrix} \ ,
	\label{expset:design:eqn:plant_dyn_output}
\end{align}
\end{subequations}
\normalsize

\medskip
where $\bm{\xi} = [\bm{p}^T \ \bm{v}^T \ \bm{b}^T \ \bm{q}^T \ \bm{\omega}]^T \in\mathbb{R}^{16}$ is the plant state vector, $(\bm{p},\bm{v})\in\mathbb{R}^3$ are the rover position and velocity, $(\bm{u}_a,\bm{u}_\omega) \in\mathbb{R}^3$ are the unknown inputs of the plant, $\bm{q}\in\mathbb{R}^4$ and $\bm{\omega}\in\mathbb{R}^3$ are the orientation quaternion and the angular velocity, and  $\bm{b}\in\mathbb{R}^3$ is a measurement bias on acceleration. No bias was considered in the angular velocity measurements. Rotational dynamics are expressed through quaternions \cite{Challa2016} with

\begin{equation}
    \bm{\Omega}(\bm{\omega}) = 
    \begin{bmatrix}
        0 & -\omega_z & \omega_y & \omega_x \\
        \omega_z & 0 & -\omega_x & \omega_y \\
        -\omega_y & \omega_x & 0 & \omega_z \\
        -\omega_x & -\omega_y & -\omega_z & 0
    \end{bmatrix} \ .
    \label{expset:design:omegamat}
\end{equation} 

\smallskip
The translational dynamics are wrapped in the substate $\bm{\xi}_p = [\bm{p}^T \ \bm{v}^T \ \bm{b}^T]^T \in \mathbb{R}^9$. Two reference frames are considered: a local frame $\mathcal{L}$ attached to the rover and a global frame $\mathcal{G}$. The rotation matrix between the two is $\mathcal{R}_{\mathcal{L}}^{\mathcal{G}}$. The state is expressed in the global frame, and the input vectors are all expressed in the local frame, i.e., $\bm{\xi} = \bm{\xi}^\mathcal{G}$, and $(\bm{u}_a,\bm{u}_\omega) = (\bm{u}_a^\mathcal{L},\bm{u}_\omega^\mathcal{L})$. The following considerations hold for this case study, with the orientation expressed in \textit{Roll, Pitch, Yaw} angles:

\begin{itemize}
    \item \textbf{Orientation control on Yaw only}: only the rotation around the z-axis is considered controllable, imposing the \textit{Yaw} angle, i.e., the user only controls $u_{\omega,3}$ in $\bm{u}_\omega$.
    \item \textbf{Terrain constraints}: \textit{Roll} and \textit{Pitch} angles are imposed by the reaction forces provided by the terrain asperities to the rover. Thus, $(u_{\omega,1},u_{\omega,2})$ in $\bm{u}$ are caused by the environment.
    \item \textbf{Torques independent on orientation}: in this scenario, external forces and torques acting on the body are not dependent on the body's orientation.
    \item \textbf{Low Angular Acceleration}: the body does not experience significant angular forces or torques due to the reduced operational speed $(\sim 1 \text{ m/s})$.
\end{itemize}

As for the system output, $\bm{y}\in\mathbb{R}^{3N+6}$ is the vector of plant measurements consisting of the acceleration and angular velocity signals $(\bm{y}_a,\bm{y}_\omega)\in\mathbb{R}^3$ provided by an IMU, and by the vector of ranging measurements $\bm{d}\in\mathbb{R}^{3N}$, introduced later. The IMU measurements $(\bm{y}_a,\bm{y}_\omega)$ are considered in the local frame $\mathcal{L}$, i.e., $(\bm{y}_a, \bm{y}_\omega) = (\bm{y}_a^{\mathcal{L}},\bm{y}_\omega^{\mathcal{L}})$. In the observer design, these quantities will be needed in the global frame $\mathcal{G}$, and thus the transformed vectors will be referred to as $(\bm{y}_a^{\mathcal{G}},\bm{y}_\omega^{\mathcal{G}})$. The transformation is computed by exploiting the rotation matrix $\mathcal{R}_\mathcal{L}^\mathcal{G}$ between the two frames, namely, $\bm{y}_a^{\mathcal{G}} = \mathcal{R}_\mathcal{L}^\mathcal{G}\bm{y}_a^{\mathcal{L}}$. The rotation matrix $\mathcal{R}_\mathcal{L}^\mathcal{G}$ computation will be addressed later in the paper. We also consider a set of ranging measurements $\bm{d} = [d_{1,z} \ \dots \ d_{N,z}]^T \in\mathbb{R}^{3N}$ provided by $N$ fixed UWB anchors at known locations, and three UWB tags installed on the agent. Each tag is identified with the index $z\in\{1,2,3\}$. Lastly, $\bm{\nu}_d\in\mathbb{R}^{3N}$, and $ (\bm{\nu}_a,\bm{\nu}_\omega)\in\mathbb{R}^3$ are the measurement noises of the UWB and the IMU. Ranging measurements are defined as $d_{i,z}^{\star} = \norm{\bm{p}_{A,i}-\bm{p}_{T,z}}$, where $\bm{p}_{A,i}\in\mathbb{R}^3$ is the known absolute position of the $i$-th anchor and $\bm{p}_{T,z}$ is the position of the z-th tag installed on the object.

\begin{assumption}[Tag positioning]
    The three UWB tags are placed on the rover such that their barycenter $\mathbb{P}\in\mathbb{R}^3$ coincides with the barycenter of the rigid body $\mathbb{G}\in\mathbb{R}^3$, namely it holds:

    \vspace{-10pt}
    \begin{equation}
        \mathbb{P}\triangleq \dfrac{1}{3}\sum\limits_{z=1}^3 \bm{p}_{T,z} \equiv \mathbb{G} \in \mathbb{R}^3 \ .
        \label{expset:design:eqn:barycenter}
    \end{equation}
    
    \label{expset:design:ass:tagpositioning}
\end{assumption}

\vspace{-10pt}
A further comment should be made on the output $\bm{y}$ sampling. In the scenario considered, $\bm{y}$ is sampled by a multi-rate acquisition hardware, that is, IMU and UWB measurements are available with different rates, $\Delta t_a$ and $\Delta t_d$, respectively defined as

\vspace{-10pt}
\begin{subequations}
	\begin{align}        
		t_{\bm{a},q}-t_{\bm{a},q-1} & = \Delta t_a \ , \\
		t_{\bm{d},l}-t_{\bm{d},l-1} & = \Delta t_d \ ,
	\end{align}
    \label{expset:design:eqn:samptimes}
\end{subequations}
\vspace{-10pt}

where $(q,l)\in\mathbb{N}$, $\Delta t_a \leq \Delta t_d$ and 
$ \Delta t_d = \bar{h} \Delta t_a$, with $\bar{h}\in\mathbb{N}_{0} \gg 1$. Indeed, the IMU acquisition can be considered continuous compared to the UWB. Lastly, as from \cite{Challa2016}, quaternions have to meet the unitary constraint. Thus, after any operation on quaternions, they are normalized.

\subsection{The plant as a hybrid system} 
\label{expset:hyb}
Due to the IMU and UWB multi-rate sampling, the plant can be considered as a hybrid system. Hybrid systems are those systems whose dynamics evolve as a pair of continuous and discrete time mappings, namely an \textit{Ordinary Differential Equation} and a \textit{Difference Equation} \cite{Goebel2009}. Every instant, only one of these mappings is responsible for the state trajectory of the system. The switch between the two is defined by the state belonging to one of two specific manifolds, i.e., the continuous or discrete time domains. Thus,  \eqref{expset:design:eqn:plant_dyn} becomes

\vspace{-5pt}
\small
\begin{subequations}
	\begin{equation}
		\mathcal{P}_h :
		\begin{aligned}
			  &
			\begin{cases}
				\dot{\bm{p}} & = \bm{v} \\ \dot{\bm{v}} & = \mathcal{R}_{\mathcal{L}}^{\mathcal{G}}\bm{u}_a \\
		\dot{\bm{b}} & = 0 \\
                    \dot{\bm{q}} & = \dfrac{1}{2}\bm{\Omega}(\bm{\omega})\bm{q}       \\
		          \dot{\bm{\omega}} & = \mathcal{R}_{\mathcal{L}}^{\mathcal{G}}\bm{u}_{\omega} \\
				\dot{\tau}   & = 1
			\end{cases}
			& \ (\bm{\xi},\tau)\in\mathcal{C} \ , \\
			  &
			\begin{cases}
				\bm{\xi}^+ & = \bm{\xi} \\
				\tau^+   & = 0
			\end{cases}
			& \ (\bm{\xi},\tau)\in\mathcal{D} \ ,
		\end{aligned}
	\end{equation}	
        \begin{equation}
            \bm{y} : \hspace{10pt}
            \begin{aligned}
                &
                \begin{cases} \quad                               
                \begin{bmatrix}
                    \bm{y}_a \\
                    \bm{y}_{\omega} 
                \end{bmatrix}
                \end{cases}
                \hspace{30pt}(\bm{\xi},\tau)\in\mathcal{C} \ , \\            
                &
                \begin{cases} \quad                            
                \begin{bmatrix}
                    \bm{y}_d \\
                    \bm{y}_a \\
                    \bm{y}_{\omega} 
                \end{bmatrix}
                \end{cases}
                \hspace{30pt} (\bm{\xi},\tau)\in\mathcal{D} \ ,
            \end{aligned}
        \end{equation}
	\label{expset:design:eqn:plant_dyn_hybrid}
\end{subequations}
\normalsize

\medskip
where sets $(\mathcal{C},\mathcal{D})\subset\mathcal{W}\times\mathbb{R}$ represent the state trajectory continuous and discrete domains, respectively. The $\tau$ variable plays the role of a virtual counter keeping track of the jumps occurring in the hybrid system state trajectory~\cite{Carnevale2007,Ferrante2016}. In the considered case, $\tau$ is reset every $\Delta t_d$ seconds.  Here, note that the jump map does not change the dynamics; only the continuous-time mapping is responsible for the system evolution. The jump map here is only used to describe the different output vectors. In hybrid systems, the state is identified by both a time instant $t$ and a jump number $j$, i.e., $\bm{\xi}(t,j) = \bm{\xi}(t_j) = \bm{\xi}_j$.

\subsection{The hybrid observer}
\label{expset:hybobs}
The hybrid system formalism is pivotal to understanding the observer dynamics.  Specifically, the proposed observer extends the one used in \cite{Sferlazza2018,Oliva2023} also to consider rotational dynamics. The observer dynamics are 

\vspace{5pt}
\footnotesize
\begin{equation}
    \hat{\mathcal{P}}_h :
    \begin{aligned}
          &
        \begin{cases}
            \dot{\hat{\bm{p}}} & = \hat{\bm{v}} \\
            \dot{\hat{\bm{v}}} & = \hat{\bm{a}} - \hat{\bm{b}} \\
            \dot{\hat{\bm{b}}} & = 0 \\
            \dot{\hat{\bm{a}}} & = \alpha(\bm{y}_a^{\mathcal{G}} - \hat{\bm{a}}) \\
            \dot{\hat{\bm{q}}} & = \dfrac{1}{2}\bm{\Omega}(\hat{\bm{\omega}})\hat{\bm{q}}  \\
            \dot{\hat{\bm{\omega}}} & = \alpha(\bm{y}_{\omega}^{\mathcal{G}} - \hat{\bm{\omega}}) \\
            \dot{\tau}       & = 1
        \end{cases}
        \hspace{40pt} \rightarrow
        \begin{cases}
    		\dot{\hat{\bm{\xi}}}_p & = \quad \begin{bmatrix}
    		    \hat{\bm{v}} \\
                \hat{\bm{a}} - \hat{\bm{b}} \\
                0 \\
                \alpha(\bm{y}_a^{\mathcal{G}} - \hat{\bm{a}})
    		\end{bmatrix}        \\
            \dot{\hat{\bm{q}}} & = \dfrac{1}{2}\bm{\Omega}(\hat{\bm{\omega}})\hat{\bm{q}}  \\
            \dot{\hat{\bm{\omega}}} & = \alpha(\bm{y}_{\omega}^{\mathcal{G}} - \hat{\bm{\omega}}) \\
    	\end{cases}
        \\ & \hspace{150pt} (\hat{\bm{\xi}},\tau)\in\hat{\mathcal{C}} \ , \\ \\
          &
        \begin{cases}
            \hat{\bm{p}}^+ & = \hat{\bm{p}} + \text{\textbf{K}}_1(g(\bm{y}_d)-\hat{\bm{p}}) \\
            \hat{\bm{v}}^+ & = \hat{\bm{v}} + \text{\textbf{K}}_2(g(\bm{y}_d)-\hat{\bm{p}}) \\
            \hat{\bm{b}}^+ & = \hat{\bm{b}} + \text{\textbf{K}}_3(g(\bm{y}_d)-\hat{\bm{p}}) \\
            \hat{\bm{a}}^+ & = \hat{\bm{a}} \\
            \hat{\bm{q}}^+ & =  \hat{\bm{q}} + \rho(\hat{\bm{q}},\bm{y}_d,\text{\textbf{K}}_4) \\
            \hat{\bm{\omega}}^+ & = \hat{\bm{\omega}} \\
            \tau^+       & = 0
        \end{cases}
        \rightarrow
        \begin{cases}
    		\hat{\bm{\xi}}_p^+ & = \hat{\bm{\xi}}_p + \tilde{\text{\textbf{K}}}(g(\bm{y}_d) - \hat{\bm{p}})         \\
            \hat{\bm{q}}^+ & =  \hat{\bm{q}} + \rho(\hat{\bm{q}},\bm{y}_d,\text{\textbf{K}}_4) \\
            \hat{\bm{\omega}}^+ & = \hat{\bm{\omega}} \\
    	\end{cases}
        \\ & \hspace{150pt}(\hat{\bm{\xi}},\tau)\in\hat{\mathcal{D}} \ ,
    \end{aligned}
\label{expset:design:eqn:obs_dyn_hybrid}
\end{equation}
\normalsize

\medskip
where $\hat{\bm{\xi}}_p = [\hat{\bm{p}}^T \ \hat{\bm{v}}^T \ \hat{\bm{b}}^T \ \hat{\bm{a}}^T]^T \in \mathbb{R}^{12}$. The full observer's vector state is $\hat{\bm{\xi}} = [\hat{\bm{p}}^T \ \hat{\bm{v}}^T \ \hat{\bm{b}}^T \ \hat{\bm{a}}^T \ \hat{\bm{q}}^T \ \hat{\bm{\omega}}]^T \in\mathbb{R}^{19}$, and it is expressed in the global frame. Furthermore, $\hat{\mathcal{C}} \triangleq \{ \hat{\xi}\in\mathbb{R}^{19} , \tau \in [0,\ \Delta t_d)\}$ and 
$\hat{\mathcal{D}} \triangleq \{ \hat{\bm{\xi}}\in\mathbb{R}^{19} , \tau \geq \Delta t_d\}$. Without loss of generality, the timer variable  $\tau$ is considered the same for all hybrid systems. The scalar $\alpha$ describes a linear filter. The mappings $(g(\cdot),\rho(\cdot))$, and the constant matrices $\text{\textbf{K}}_i \in \mathbb{R}^{3\times 3}$ will be detailed later. Before that, $(\bm{\nu}_f,\bm{\nu}_g)$ signals are defined in order to rewrite $(\hat{\bm{a}},\hat{\bm{\omega}})$ as

\vspace{-5pt}
\begin{equation}
    \label{expset:design:eqn:vfvg}
    \begin{aligned}
        \hat{\bm{a}} &= \mathcal{R}_{\mathcal{L}}^{\mathcal{G}}\bm{u}_a + \bm{b} + \bm{\nu}_f \ , \\
        \hat{\bm{\omega}} &=\mathcal{R}_{\mathcal{L}}^{\mathcal{G}}\bm{\omega} + \bm{\nu}_g \ . \\
    \end{aligned}
\end{equation}

Specifically, $(\bm{\nu}_f,\bm{\nu}_g)$ take into account the mismatch between $\bm{a}$
and $\hat{\bm{a}}$, and $\bm{\omega}$
and $\hat{\bm{\omega}}$, introduced by the linear dynamics in  
\eqref{expset:design:eqn:plant_dyn}-\eqref{expset:design:eqn:obs_dyn_hybrid} and the Gaussian noise measurement $(\bm{\nu}_a,\bm{\nu}_\omega)$ in  \eqref{expset:design:eqn:plant_dyn_output}. As $(\hat{\bm{a}},\hat{\bm{\omega}})$ are expressed in the global reference frame, the injection terms in the observer will use $(\bm{y}_a^{\mathcal{G}},\bm{y}_\omega^{\mathcal{G}})$.

\subsection{The position trilateration}
\label{expet:gmap}
The terms introduced in the observer structure are now investigated, starting from the nonlinear map $g(\cdot):\mathbb{R}^{3N}\rightarrow \mathbb{R}^3$,
which represents the output of a trilateration algorithm providing an estimate of the rover's absolute position $\bar{\bm{p}}$ exploiting the ranging measurements. The trilateration problem is defined as

\vspace{-5pt}
\small
\begin{subequations}
    \begin{equation}
        \{\bar{\bm{p}}_{T,z}\} = 
        \begin{aligned}        	
    		  \underset{\{\bm{p}_{T,z}\}}{\arg\min}
    		  \quad \underbrace{\sum\limits_{i=1}^N\sum\limits_{z=1}^3(\norm{\bm{p}_{A,i}-\bm{p}_{T,z}}-d_{i,z})^2}_{\text{J}} \ ,
        \end{aligned}         
        \label{expset:design:eqn:pd_optproblem}
    \end{equation}    
    \begin{equation}
        \bar{\bm{p}} = \dfrac{1}{3}\sum\limits_{z=1}^3 \bar{\bm{p}}_{T,z} \ .     
        \label{expset:design:eqn:pd_average}        
    \end{equation}    
\end{subequations}
\normalsize

\medskip
Indeed, the solution to  \eqref{expset:design:eqn:pd_optproblem} can be computed either in closed form or numerically \cite{Zhou2009}. The Newton-Raphson algorithm efficiently solves the trilateration problem with local quadratic convergence to (one of the) minima of $\text{J}$. A single-step evaluation at the i-th iteration is

\vspace{-10pt}
\begin{equation}
   \begin{aligned}
   \bm{p}_{T,z,i+1} = \bm{p}_{T,z,i} - \big(H_\text{J}(\bm{p}_{T,z,i})\big)^{-1} \nabla\text{J}(\bm{p}_{T,z,i}) \ ,
   \end{aligned}
   \label{expset:design:eqn:eqn:Newton-Rapson}
\end{equation}

\medskip
where $H_\text{J}$ is the Hessian matrix of $\text{J}$ and $\nabla \text{J}$ its gradient, both calculated with respect to $\bm{p}_{T,z}$. The Hessian needs to be inverted over the trajectories considered. It can be easily shown that the singularity condition of the Hessian occurs when the UWB tags reach exactly the location of any fixed anchors, namely $\bm{p}_{T,z}=\bm{p}_A$, which never happens in this work. In this experimental setup, iterating no more than five times, the Newton-Raphson algorithm from  \eqref{expset:design:eqn:eqn:Newton-Rapson} yields an accurate, although approximated, solution of problem \eqref{expset:design:eqn:pd_optproblem} for each UWB tag, which will be referred to as $\bar{\bm{p}}_{T,z,5}$, with $z\in\{1,2,3\}$. The average of these values as for   \eqref{expset:design:eqn:pd_average} will be $\bar{\bm{p}}_5$, on which the following assumption is based:

\begin{assumption}[Newton-Rhapson convergence]
 Let $ \bar{\bm{p}}_5\in \Omega_p\subset\mathbb{R}^3$, with $\Omega_p$ and $\bm{\nu}_d$  bounded. Define $g(\bm{y}_d)\triangleq\bar{\bm{p}}_5$, where $\bar{\bm{p}_5}$ is obtained from  \eqref{expset:design:eqn:pd_average}. Furthermore,  $\{\bar{\bm{p}}_{T,z,5}\}$ is obtained after five iterations of the Newton-Raphson algorithm from  \eqref{expset:design:eqn:eqn:Newton-Rapson},  with initial conditions $\{\bar{\bm{p}}_{T,z,0}\}$. Then,

    \vspace{-10pt}
    \begin{equation}
        \bar{\bm{p}}_5 =  \bm{p} + \bm{\nu}_J(\bm{\nu}_{d}) \ ,
        \label{expset:design:eqn:pj_noise}
    \end{equation}

holds true for any $\{\bar{\bm{p}}_{T,z,0}\}\in \Omega_p$, with a bounded function 
$\bm{\nu}_J:\mathbb{R}^{3N}\rightarrow\mathbb{R}^3$ that takes into account 
the uncertainties introduced by the UWB measurement noise $\bm{\nu}_d$, and the approximate
solution of  \eqref{expset:design:eqn:pd_optproblem}.
\label{expset:design:ass:pJ_noise}
\end{assumption}

In light of this, it is now clear that the term $(g(\bm{y}_d) - \hat{\bm{p}})$ represents a mismatch between the position estimate from the trilateration and the current position estimate of the observer. The static gain $\text{\textbf{K}} \triangleq \begin{bmatrix} \text{\textbf{K}}_1 \ \text{\textbf{K}}_2 \ \text{\textbf{K}}_3 \end{bmatrix}^T \in\mathbb{R}^{9 \times 3}$, together with the output mismatch, act as output injection terms of the observer through $\tilde{\text{\textbf{K}}}\in\mathbb{R}^{12 \times 3}$, with $\tilde{\text{\textbf{K}}} = \begin{bmatrix} \text{\textbf{K}}^T \ 0_{3 \times 3} \end{bmatrix}^T$. Position correction occurs only on the jump map when UWB measurements are collected. 

\subsection{Computing the orientation} 
\label{expset:rho}
The orientation estimation is now considered. The injection term is designed here as a nonlinear mapping $\rho(\hat{\bm{q}},\bm{y}_d,\text{\textbf{K}}_4)$, and it depends on the UWB measurements $\bm{y}_d$, the current quaternion estimate $\hat{\bm{q}}$, and a static gain $\text{\textbf{K}}_4$, to be designed:

\vspace{-5pt}\small
\begin{equation}
    \label{expset:design:eqn:injtermQ}
    \rho(\hat{\bm{q}},\bm{y}_d,\text{\textbf{K}}_4) \triangleq \text{A2Q}\bigg(\text{\textbf{K}}_4(\delta(\bm{y}_d)-\text{Q2A}(\hat{\bm{q}}))\bigg) \ .
\end{equation}
\normalsize

Here, A2Q, Q2A, and $\delta$ are these nonlinear mappings:

\begin{itemize}
    \item \textbf{Q2A} $(\mathbb{R}^{4} \rightarrow \mathbb{R}^{3})$: identifies the mapping from the quaternion space to Euler angles, namely the rotation around the X-axis (\textit{Roll} angle $\zeta$), the Y-axis (\textit{Pitch} angle $\theta$), and the Z-axis (\textit{Yaw} angle $\varphi$),

    \item \textbf{A2Q} $(\mathbb{R}^{3} \rightarrow \mathbb{R}^{4})$: identifies the mapping from Euler angles to the quaternion space. Note that A2Q and Q2A are one the inverse of the other, namely $\text{A2Q}(\text{Q2A}(\bm{q})) = \bm{q}$ and $\text{Q2A}(\text{A2Q}(\zeta,\theta,\varphi)) = (\zeta,\theta,\varphi)$. 
    
    \item $\bm{\delta} \ (\mathbb{R}^{9}\times\mathbb{R}^{9} \rightarrow \mathbb{R}^{3})$: identifies the mapping from three non-collinear points of a rigid body to its orientation expressed in Euler angles. As previously introduced, the orientation of a rigid body in space can be determined if the coordinates of at least three non-collinear points placed on it are known. Specifically, this task can be described as a \href{https://simonensemble.github.io/posts/2018-10-27-orthogonal-procrustes/}{Procrustes problem}\cite{Procrustes} and consists of the equivalent of the trilateration problem but for the orientation. The $\delta(\bm{y}_d)$ term solves it, and assumes the following:

    \vspace{-10pt}
   \begin{assumption}[Procrustes convergence]
       \label{expset:design:ass:noisetrilateration}
       Let $\bm{p}\in\Sigma_p\subset\mathbb{R}^3$ and $\bm{q}\in\Sigma_q\subset\mathbb{R}^4$ be the true position and orientation of the nominal plant $\mathcal{P}_h$. Let the measurement noise $\bm{\nu}_d$ be bounded on the UWB measurements. Define $g(\bm{y}_d) \triangleq \bar{\bm{p}}_5$ as from \cref{expset:design:ass:pJ_noise}. Define also the orientation estimate $\bm{q}_5$ provided by the orientation trilateration $\delta(\bm{y}_d)$. The following is assumed:

        \vspace{-10pt}
        \begin{subequations}
            \begin{equation}                
                \bar{\bm{p}}_5 = \bm{p} + \bm{\nu}_p(\bm{\nu}_d )\ ,  
            \end{equation}  
            \vspace{-25pt}
            \begin{equation}
                \bar{\bm{q}}_5 = \bm{q} + \bm{\nu}_q(\bm{\nu}_d) \ ,
            \end{equation}
            \label{expset:design:eqn:asstrilateration}
        \end{subequations}

        \vspace{-5pt}
        with $(\bm{\nu}_p) \ : \ \mathbb{R}^{3N} \rightarrow \mathbb{R}^3$ and $(\bm{\nu}_q) \ : \ \mathbb{R}^{3N} \rightarrow \mathbb{R}^4$ bounded functions where $\bm{\nu}_p(0) = 0$ and $\bm{\nu}_q(0) = 0$.
   \end{assumption}
   
   According to this last assumption, the injection term from  \eqref{expset:design:eqn:injtermQ} depends on $(\text{Q2A}(\bm{q} + \bm{\nu}_q(\bm{\nu}_d))\ - \ \text{Q2A}(\hat{\bm{q}}))$ which represents a mismatch between the orientation obtained from the UWB measurements and the orientation estimate.

   \begin{remark}
       \label{expset:design:rem:rhozeroineqzero}
       If no measurement noise is present, i.e., $\bm{v}_q(\bm{v}_d) = 0$, $\rho(\hat{\bm{q}},\bm{y}_d,\text{\textbf{K}}_4)$ results zero in $\bm{q} = \hat{\bm{q}}$.
   \end{remark}      
   Thus, similarly to the position, the static gain $\text{\textbf{K}}_4 \in \mathbb{R}^{3\times 3}$ defines the injection term, which again belongs to the jump map. The correction in Euler Angles reads

    \small
    \vspace{-10pt}
   \begin{align}
        \label{expset:design:eqn:injtermQTrilater}
        &\hspace{-30pt}\rho(\hat{\bm{q}},\bm{y}_d,\text{\textbf{K}}_4) \triangleq \text{A2Q}(\text{\textbf{K}}_4(\text{Q2A}(\bm{q} + \bm{\nu}_q(\bm{\nu}_d)) - \text{Q2A}(\hat{\bm{q}}))).
    \end{align}
    \normalsize
\end{itemize}

It is important to note that the rover orientation returned by this procedure is used to define the rotation matrix from the rover's local reference frame to the global reference frame, i.e., $\mathcal{R}_\mathcal{L}^\mathcal{G}$. This is how the IMU measurements are transformed into the global reference frame every time they are collected, namely $\bm{y}_a^\mathcal{G} = \mathcal{R}_\mathcal{L}^\mathcal{G}\bm{y}_a$, and $\bm{y}_\omega^\mathcal{G} = \mathcal{R}_\mathcal{L}^\mathcal{G}\bm{y}_\omega$.

\medskip
The observer keeps its hybrid structure with the correction terms in the jump map. Following the notation introduced in \cref{meth} for the TBOD method, the design parameters of the observer are $\bm{\Delta} = \{\text{\textbf{K}}_1 \ \text{\textbf{K}}_2 \ \text{\textbf{K}}_3 \ \text{\textbf{K}}_4\}$, and $\bm{\Gamma} = \emptyset$. More specifically, the set of parameters boils down to six scalar variables, i.e., $\{k_i\} \text{ s.t. } i\in\{1,\dots,6\}$:

\begin{equation}
    \label{expset:design:gainstruct}
    \text{\textbf{K}}_i = 
    \begin{cases}
        k_i \text{\textbf{I}}_3 \quad  &\in \mathbb{R}^{3\times 3} \ \forall i\in \{1,2,3\} \\
        \\
        \begin{bmatrix}
            k_4 &  0 & 0 \\
            0 & k_5 & 0 \\
            0 & 0 & k_6
        \end{bmatrix}  \quad &\in \mathbb{R}^{3\times 3} \ \forall i\in \{4,5,6\}
    \end{cases}.
\end{equation}

The gains on the position mismatch are used for the injection terms of the position, velocity, and acceleration bias estimate. A three-dimensional vector describes each of these. Instead, the orientation gains are only used to update the quaternion estimate, resulting in a scalar term for each related Euler angle. To conclude this section, the estimation error from $\mathcal{P}_h$ and $\hat{\mathcal{P}_h}$ is defined as

\begin{equation}
    \label{expset:design:eqn:errdef}
    \bm{e} \triangleq 
    \begin{bmatrix}
        \bm{p} - \hat{\bm{p}} \\
        \bm{v} - \hat{\bm{v}} \\
        \bm{b} - \hat{\bm{b}} \\
        \bm{q}^{-1} \otimes \hat{\bm{q}} \\
    \end{bmatrix}
    = 
    \begin{bmatrix}
        \bm{e}_p \\
        \bm{e}_q
    \end{bmatrix}
    \in\mathbb{R}^{13},
\end{equation}

\smallskip
where the terms $(\bm{e}_p,\bm{e}_q)$ are the position and orientation errors, respectively. Considering a general quaternion with the scalar term as the last, the quaternion error is defined following  \cite{Challa2016}, exploiting the inversion and product operations. Specifically, the product between $(\bm{q},\hat{\bm{q}})$ is defined as $\bm{q}\otimes\hat{\bm{q}} = M(\bm{q})\hat{\bm{q}}$, with $M(\bm{q})$ is a skew-symmetric matrix (see \cite{Challa2016}). It is easier to interpret the orientation estimation error with Euler angles; thus, the quaternion error is converted into Euler angles through Q2A in the results section. 

\subsection{Stability analysis}
\label{expset:stab}

The stability of the estimation error is here analyzed when the observer $\hat{\mathcal{P}}_h$ in \cref{expset:design:eqn:obs_dyn_hybrid} is used to track the plant $\mathcal{P}_h$ in \eqref{expset:design:eqn:plant_dyn_hybrid}. Specifically, as the stability of a hybrid system is considered, the concept of \textit{Uniformly Ultimately Boundedness} from \cite{Goebel2009} will be exploited. Similarly to the position-only case studied in \cite{Oliva2023}, the estimation error's dynamics are:

\small
\begin{equation}
    \label{expset:eqn:errDyn}
    \mathcal{E}_h = 
    \begin{aligned}
    &
        \begin{cases}
            \dot{\bm{e}}_p &= 
            \begin{bmatrix}
                \dot{\bm{p}} - \dot{\hat{\bm{p}}} \\
                \dot{\bm{v}} - \dot{\hat{\bm{v}}} \\
                \dot{\bm{b}} - \dot{\hat{\bm{b}}} \\
            \end{bmatrix} =
            A_e\bm{e}_p + B_e\bm{\nu}_f\\
            \dot{\bm{e}}_q &= \dot{M}(\bm{q}^{-1})\hat{\bm{q}} + M(\bm{q}^{-1})\dot{\hat{\bm{q}}}
        \end{cases}, \\ & \hspace{150pt}
        (\bm{e},\tau)\in\mathcal{C}_e \ , \\
        & 
        \begin{cases}
            \bm{e}_p^+ &= 
            \begin{bmatrix}
                \bm{p}^+ - \hat{\bm{p}}^+ \\
                \bm{v}^+ - \hat{\bm{v}}^+ \\
                \bm{b}^+ - \hat{\bm{b}}^+ \\
            \end{bmatrix}
            = (\text{\textbf{I}} - \bm{\Gamma}_K)\bm{e}_p + \text{\textbf{K}}\bm{\nu}_f \\
            \bm{e}_q^+ &= 
                (\bm{q}^+)^{-1}\otimes\hat{\bm{q}}^+            
             = \bm{e}_q + M(\bm{q}^{-1})\rho(\hat{\bm{q}},\bm{y}_d,\text{\textbf{K}}_4)
        \end{cases}, \\ & \hspace{150pt} (\bm{e},\tau)\in\mathcal{D}_e \ ,
    \end{aligned}
\end{equation}
\normalsize

where $\mathcal{C}_e \triangleq \{ \bm{e}\in\mathbb{R}^{13} , \tau \in [0,\ \Delta t_d]\}$ and 
$\mathcal{D}_e \triangleq \{ \bm{e}\in\mathbb{R}^{13} , \tau \geq \Delta t_d\}$. The proof of the following result regarding the stability of the error dynamics in \eqref{expset:eqn:errDyn} is now provided. 

\begin{theorem}[Full observer - stability]
\label{expset:the:boundedness}
Consider the hybrid plant $\mathcal{P}_h$ in  \eqref{expset:design:eqn:plant_dyn_hybrid} and observer in  \eqref{expset:design:eqn:obs_dyn_hybrid}, where assumption \ref{expset:design:ass:noisetrilateration} holds. Consider the noise-free scenario, i.e., $(\bm{\nu}_a, \bm{\nu}_\omega, \bm{\nu}_d) = 0$. Assume that there exists a constant $\overline{M}\in\mathbb{R}_{>0}$ such that:

\vspace{-5pt}
\begin{equation}
    \label{expset:eqn:theoassumptionMoverline}
    \max\{ \norm{u_a}, \norm{u_\omega}, \norm{\dot{u}_a}, \norm{\dot{u}_\omega} \} \leq \overline{M} \ .
\end{equation}

Then, the origin of the estimation error system $\mathcal{E}_h$ in  \eqref{expset:eqn:errDyn} is Uniformly Globally pre-Asymptotically stable\footnote{The analog of the GAS property for hybrid systems.}\cite{Goebel2009}.
    
\end{theorem}

\begin{proof}
    
    The boundedness assumption in  \eqref{expset:eqn:theoassumptionMoverline} implies $(\bm{y}_\omega,\bm{y}_a)$ to be bounded. Hence, also $(\hat{\bm{\omega}},\hat{\bm{a}})$ are bounded as they are states of BIBO systems, fed by $(\bm{y}_\omega,\bm{y}_a)$. $(\bm{\omega},\bm{a})$ are bounded too as per their definition in \eqref{expset:design:eqn:plant_dyn}.

    \medskip
    Let's focus now on the dynamics of $\bm{e}_q$ defined in \eqref{expset:eqn:errDyn}, and first show that this subsystem is \textit{Uniformly Globally pre-Asymptotically Stable} (UGpAS) \cite{Goebel2009} in the nominal case, i.e., without measurement noise  $(\bm{\nu}_a,\bm{\nu}_\omega,\bm{\nu}_d) = 0$. The set of interest is $\mathcal{A} \triangleq \{\bm{e}_q\in\mathbb{R}^4,\bm{e}_p\in\mathbb{R}^9 \ | \ \bm{e}_q = \bm{0}_q, \bm{e}_p = \bm{0}\}$, namely the origin of the error dynamics \eqref{expset:eqn:errDyn}, where $\bm{0}_q = [0,0,0,1]^T$. Consider the quadratic form

    \vspace{-5pt}\small
    \begin{align}
        \label{expset:eqn:LyapFun}
        V(\bm{e}_q) &\triangleq \bm{e}_q^{T}\bm{e}_q,
    \end{align}
    \normalsize
    
    such that $V(\bm{e}_q) > 0 \ \forall \bm{e}_q\neq 0$, and $V(\bm{e}_q) = 0$ for $\bm{e}_q = \bm{0}_q$. As both $(\bm{\omega},\hat{\bm{\omega}})$ are bounded, also the quaternions $(\bm{q},\hat{\bm{q}})$ are bounded. Indeed, the orientation error is also limited, specifically $\norm{\bm{e}_q} \leq 1$, as it is a quaternion. Let's now check the boundedness of $\dot{V}$ in the first flow time interval $[0, \ t]$. From \eqref{expset:eqn:errDyn} it holds that

    \vspace{-5pt}\small
    \begin{equation}
        \label{expset:eqn:eqdotLimit}
        \norm{\dot{\bm{e}}_q} = \norm{\dot{M}(\bm{q}^{-1})\hat{\bm{q}} + M(\bm{q}^{-1})\dot{\hat{\bm{q}}}} \leq \overline{C} > 0 \ ,
    \end{equation}
    \normalsize
    
    as $(M,\bm{q},\hat{\bm{q}})$ are all bounded. Recalling that $\norm{\bm{e}_{q,0}} = 1$, within the flow time interval $[0, t]$ the following holds:

    \vspace{-5pt}
    \small
    \begin{align}
        \label{expset:eqn:eqLimit}
        \norm{\bm{e}_q} &\leq \norm{\bm{e}_{q,0}} + \int\limits_0^t\overline{C}d\tau = \norm{\bm{e}_{q,0}} + t\overline{C} \leq \notag \\
        &\leq \norm{\bm{e}_{q,0}}\bigg(1 + \dfrac{t\overline{C}}{\norm{\bm{e}_{q,0}}}\bigg) \leq \norm{\bm{e}_{q,0}}(1 + t\overline{C}).
    \end{align}
    \normalsize
    
    Indeed, this inequality can be used to provide a bound on the Lyapunov function in $[0, \ t]$: 

    \begin{equation}
        \label{expsetup:eqn:LyapBound}
        V(\bm{e}_q) \leq \norm{\bm{e}_q}^2 \leq \norm{\bm{e}_{q,0}}^2(1+\overline{C}t)^2 \ .
    \end{equation}

    The $\dot{V}$ computation in $[0, \ t]$ provides:  

    \begin{align}
    \label{expset:eqn:LyapFunGradFlow}
        \langle\nabla V(\bm{e}_q),\dot{\bm{e}}_q\rangle &= \dot{\bm{e}}_q^T\bm{e}_q + \bm{e}_q^T\dot{\bm{e}}_q \leq  \notag \\
        & \leq 2\overline{C}\norm{\bm{e}_{q,0}}(1 + t\overline{C}) \leq 2\overline{C},
    \end{align}    

    \medskip
    where the last inequality holds since in general $\norm{\bm{e}_{q}} \leq 1$, as highlighted in  \eqref{expset:eqn:eqLimit}. Considering the hybrid error dynamics in  \eqref{expset:eqn:errDyn}, these last considerations can be extended to a generic flow interval $[t_j, \ t_j + \Delta t_d]$. Thus,  \eqref{expset:eqn:LyapFunGradFlow} provides an upper bound for $\dot{V}(\bm{e}_q)$, from which the time evolution of $V(\bm{e}_q)$ cam be computed in the flow time interval $[t_j, \ t_j + \Delta t_d]$, finding an upper bound:
    \begin{align}
      \label{expset:eqn:LyapFunGradFlowLagrange}
        V(\bm{e}_q(t_j+\Delta t_d)) &\leq V(\bm{e}_q(t_j)) + \int\limits_{t_j}^{t_j+\Delta t_d}2\overline{C}d\tau = \notag \\
        &= V(\bm{e}_q(t_j)) + \Delta t_d 2\overline{C} \ .
    \end{align}

    \medskip    
    When a jump occurs, the following holds with $\rho$ from  \eqref{expset:design:eqn:injtermQTrilater}:   

    \begin{equation}
        \label{expset:eqn:LyapFunJumpInjectionNoNoise}
        \bm{e}_q^+ = \bm{e}_q +  M(\bm{q}^{-1})\rho(\bm{e}_q,\bm{y}_d,\text{\textbf{K}}_4).
    \end{equation}
   
    As $\rho \in C^1$ and $\rho(\bm{0},\bm{q},\text{\textbf{K}}_4) = \bm{0}$, then by section 4.3 of \cite{Khalil2014}, there exists a continuous function $\bar{\rho}$ such that 

     \begin{equation}        \label{expset:eqn:LyapFunJumpInjectionNoNoiseLinear}
        \rho(\bm{e}_q,\bm{y}_d,\text{\textbf{K}}_4) = \overline{\rho}(\bm{e}_q,\bm{q},\text{\textbf{K}}_4)\bm{e}_q \text{ with } \overline{\rho}\in\mathbb{R}^{4\times 4} \ ,
    \end{equation}

    \
    which results in 

    \begin{equation}
        \label{expset:eqn:LyapFunJumpInjectionNoNoiseWrap}
        \bm{e}_q^+ = \underbrace{\bigg(\text{\textbf{I}} +  M(\bm{q}^{-1})\overline{\rho}(\bm{e}_q,\bm{q},\text{\textbf{K}}_4) \bigg)}_{\Psi}\bm{e}_q  \ .
    \end{equation}
    
    To complete the stability analysis, consider the difference between $V(\bm{e}_q)$ calculated at the beginning of the general interval $t_j$, and after the jump in $t_j+\Delta t_d$. It holds:

    \small
    \begin{align}
        \label{expset:eqn:LyapFunGradJumpDef}
        &V(\bm{e}_{q,t_j+\Delta t_d}^+) - V(\bm{e}_{q,t_j}) = \notag \\
        &= \bm{e}_{q,t_j+\Delta t_d}^{+^T}\bm{e}_{q,t_j+\Delta t_d}^+ - \bm{e}_{q,t_j}^T\bm{e}_{q,t_j} = \notag \\
        & = \bm{e}_{q,t_j+\Delta t_d}^T\Psi^T\Psi\bm{e}_{q,t_j+\Delta t_d} - \norm{\bm{e}_{q,t_j}}^2 \leq  \notag \\
        &\leq \norm{\bm{e}_{q,t_j+\Delta t_d}}^2\lambda_{\max}(\Psi^T\Psi) - \norm{\bm{e}_{q,t_j}}^2 \leq \notag \\
        &\leq \norm{\bm{e}_{q,t_j}}^2(1+\overline{C}\Delta t_d)^2 \lambda_{\max}(\Psi^T\Psi) - \norm{\bm{e}_{q,t_j}}^2 = \notag \\
        &= \norm{\bm{e}_{q,t_j}}^2\bigg((1+\overline{C}\Delta t_d)^2 \lambda_{\max}(\Psi^T\Psi) - 1\bigg).
    \end{align}
    \normalsize

    \medskip
    Thus, it can be concluded that $(\bm{e}_q,\bm{e}^+_q)$ is Uniformly Globally pre-Asymptotically Stable (GAS) as long as it holds

    \begin{equation}
        \label{expset:eqn:LyapFunCond}
        \bigg((1+\overline{C}\Delta t_d)^2 \lambda_{\max}(\Psi^T\Psi) - 1\bigg) < 0 \ .
    \end{equation}

    Indeed, without measurement noise, the orientation obtained by solving the Procrustes problem is exact. Thus, an update gain $\text{\textbf{K}}_4 = \text{\textbf{I}}$ resets the estimated orientation to the true value, i.e., $\hat{\bm{q}}^+ = \bm{q}$, and hence $\bm{e}_q^+ = \bm{0}_q$. More formally, in the presence of noise, $\bm{e}_q^+ \neq  \bm{0}_q$, and hence it holds that

    \vspace{-10pt}
    \begin{align}
        \label{expset:eqn:LyapFunK4Cond}
        \exists \quad &\text{\textbf{K}}_4 \text{  s.t.  } \lambda_{\max}(\Psi^T\Psi)= 0.
    \end{align}
    
    So far, this proof has shown that the error dynamics in  \eqref{expset:eqn:errDyn} evolve with a bounded increase during the flow map (see  \eqref{expset:eqn:LyapFunGradFlowLagrange}), and that there exists an update gain $\text{\textbf{K}}_4$ such that its total decrease after the jump drives the error $\bm{e}_q$ to zero in a single jump. Thus, it can be claimed that $\bm{e}_q = \bm{0}_q$ is UGpAS without measurement noise and that it manifests a \textit{finite-time convergence}, specifically of one sample in this case.

    \medskip
    Moving on with the position error dynamics, note that position and orientation dynamics are coupled through the rotation matrix $\mathcal{R}_{\mathcal{L}}^{\mathcal{G}}$. Considering the \textit{finite-time convergence} property of the orientation error, a reliable estimation of $\mathcal{R}_{\mathcal{L}}^{\mathcal{G}}$ is considered available after one sample of the ranging sensors. Thus, there is a \textit{time separation principle} between the position and orientation error dynamics. Specifically, after one sample, the position error dynamics are decoupled from the orientation error. To conclude, the very same stability proof from \cite{Oliva2023} holds for $\bm{e}_p$, namely the error subsystem considering $(\bm{e}_p,\bm{e}_p^+)$ is also Uniformly Globally pre-Asymptotically Stable, concluding the proof.    
    
\end{proof}

\begin{corollary}[Noise scenario]
    Consider now a non-zero measurement noise, i.e., $(\bm{\nu}_a,\bm{\nu}_\omega,\bm{\nu}_d)\neq 0$. Since all the mappings considered in the error dynamics \eqref{expset:eqn:errDyn} are continuously differentiable, by the robustness property of the Lyapunov stability analysis, it can be stated that there will always be an $\bar{\epsilon} > 0$ small enough such that $(\bm{\nu}_a,\bm{\nu}_\omega,\bm{\nu}_d) < \bar{\epsilon}$ and hence, such that $\mathcal{A}$ is \textit{Uniformly Practically Asymptotically Stable}.
\end{corollary}

It is important to stress that, now that the convergence of the error dynamics has been proven, \eqref{expset:eqn:LyapFunK4Cond} is not explicitly exploited in the design procedure but rather relies on the TBOD heuristic optimization process. 

% RESULTS
\section{Results}
\label{res}

The position and orientation observer designed in \cref{expset} with the TBOD method is now tested and analyzed for its performance compared to a standard EKF.  

\subsection{Hardware and sensor setup} 
\label{res:hw}
In the considered experimental setup, a Jackal UGV \href{https://clearpathrobotics.com/jackal-small-unmanned-ground-vehicle/}{(Clearpath LTD)}\cite{Clearpath} is used. The technical specifications in the datasheet are detailed in \cref{res:tab:jackal}. The following sensors have been used: 

\begin{table}[b]
    \centering
    \begin{tabular}{|c|c|}
        \bottomrule 
        \textbf{Parameter} & \textbf{Value} \\
        \toprule
        \bottomrule
        External dimensions &  508 x 430 x 250 mm \\
        Internal dimensions & 250 x 100 x 85 mm \\
        Weight & 17 Kg \\
        Maximum payload & 20 Kg \\
        Max speed & 2.0 m/s \\
        Run time & 4 hours \\
        \toprule 
    \end{tabular}
    \caption{Technical specifications of the Clearpath Jackal.\vspace{-10pt}}
    \label{res:tab:jackal}
\end{table}

\begin{itemize}
    \item \textbf{IMU}: the acceleration and angular velocity measurements were provided by a \href{https://www.microstrain.com/content/3dm-gx3-25}{3DM-GX3-25 IMU}\cite{MicroStrain}  (MicroStrain\textsuperscript{TM,}). The measurement noise can be modeled as zero-mean Gaussian distributed with $\sigma_\omega = 0.01 \text{ rad/s}$ and $\sigma_a = 0.05 \text{ m/s}^2$ for the angular velocity and acceleration, respectively. An additive bias also affects the acceleration in the order of $0.1 \text{ m/s}^2$. Thus, $\nu_a \in \mathcal{N}(0,\sigma_a) \text{ m/s}^2$, $b\in\mathcal{N}(0,0.1) \text{ m/s}^2$, and $\nu_\omega \in \mathcal{N}(0,\sigma_\omega) \text{ rad/s}$. The sampling time is $\Delta t_a = 0.01s$.
    \item \textbf{UWB}: the distances between the tags and the anchors were provided by the \href{https://www.qorvo.com/products/p/MDEK1001}{MDEK1001}\cite{QorvoMDEK1001} development kit (Qorvo\textsuperscript{TM}), based on the \href{https://www.qorvo.com/products/p/DWM1001-DEV}{DWM1001-DEV} antennas \cite{QorvoDWM1001-DEV}. The noise on the distance measurements can be modeled as zero-mean Gaussian distributed with $\sigma_d = 0.2 \text{ m}$, i.e.,  $\nu_d \in \mathcal{N}(0,\sigma_d) \text{ m}$ The sampling time is $\Delta t_d = 0.05s$.
    \item \textbf{Ground Truth}: A set of six \href{https://www.vicon.com/}{Vicon\textsuperscript{TM}}\cite{Vicon} T10S cameras provide a position and orientation measure of the Jackal with a precision of millimeters. This measure has been considered the ground truth for the UGV and has been exploited in the TBOD design method. The sampling time is $\Delta t_a = 0.01s$, like the IMU. 
\end{itemize}

\begin{figure}[t!]        
        \centering
	\includegraphics[width=6cm]{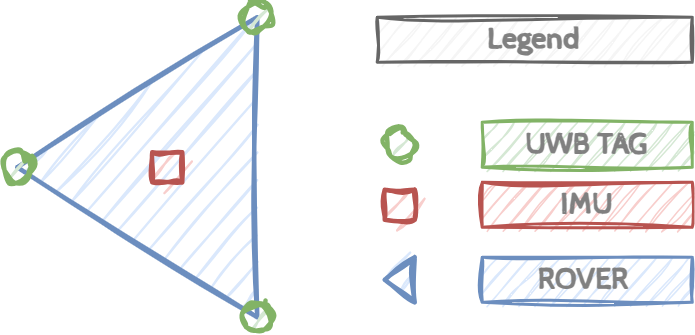}
	\caption{Graphical representation of the rovers' sensor setup.}
	\label{res:fig:roverAbove}
\end{figure}

\begin{table}[!h]
    \centering
    \begin{tabular}{|c|c|c|c|}
        \bottomrule
         \textbf{Anchor} & \textbf{X} & \textbf{Y} & \textbf{Z} \\
         \toprule
         \bottomrule
         A1 & -0.40 m & +4.20 m & +2 m \\
         A2 & -0.40 m & -1.80 m & +2 m \\
         A3 & +2.48 m & -2.20 m & +2 m \\
         A4 & +2.80 m & -4.20 m & +2 m \\
         \toprule 
    \end{tabular}
    \caption{Anchor positioning for the indoor experiment.\vspace{-10pt}}
    \label{res:tab:anchors}
\end{table}

Regarding the setup of the rover, the experiment is similar to that considered in \cite{Oliva2023}, where the anchors are placed as in \cref{res:tab:anchors}. The tags have been placed on an equilateral triangle with a side $20 \text{ cm}$, according to \cref{expset:design:ass:tagpositioning}. A schematic is reported in \cref{res:fig:roverAbove}. As previously introduced, a ground truth system is available for the position and orientation provided by the Vicon cameras. Thus, the optimization problem in  \eqref{meth:eqn:TBOD_prob_def_dist_groundtruth} is considered. An example of two experiment trajectories is reported in \cref{res:fig:rover3D}. Specifically, the red-dashed trajectory belongs to the training set $\mathcal{T}$, while the solid-gray trajectory is one of those used to test the observer's performance. The bottom-right subplot shows the 3D perspective, while the others depict the 2D projections. The blue circles are the anchors, while the grey lines are the rover trajectory. The altitude Z changes are quite small due to space constraints in the Lab. Hence, the orientation changes are mainly in the Yaw angle. Regarding the system architecture, the project was developed on Ubuntu 18.04 and ROS Melodic.

\begin{figure}[t!]        
        \centering         
	\includegraphics[trim={0 20pt 0 50pt},clip,width=0.5\textwidth]{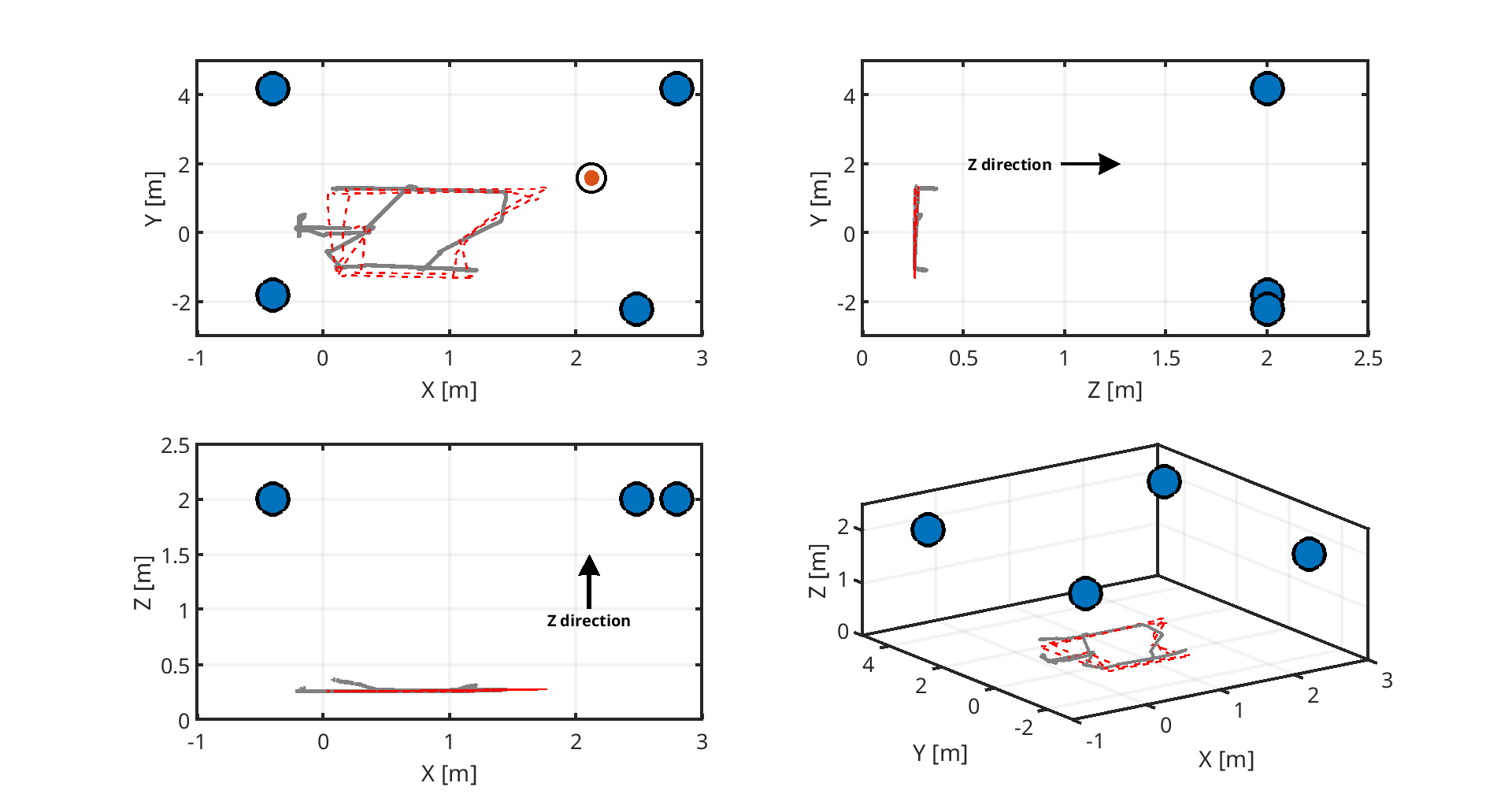}
	\caption{Representation of one of the trajectories used for the tuning (red-dashed) and one of the trajectories used to test the observer performance (solid-gray). The blue-filled circles are the anchors.}\vspace{-10pt}
	\label{res:fig:rover3D}
\end{figure}

\subsection{Tuning and results} 
\label{res:res}
The $\{k_1,\dots,k_6\}$ gains of the hybrid observer in ~\eqref{expset:design:eqn:obs_dyn_hybrid} have been tuned by solving ~\eqref{meth:eqn:TBOD_prob_def_dist_groundtruth} on a set of $M=4$ trajectories. The values obtained are reported in \cref{res:tab:gains}. Notice that the belief in the output of the trilateration and the Procrustes problem solution is pretty high, almost reaching a full reset on the position and orientation. The optimization was performed offline using a dataset of approximately 300 seconds per trajectory, requiring about 10 hours on a standard laptop (Intel i7, 32GB RAM), demonstrating the method's practicality for overnight tuning in typical research or field settings.

\medskip
These values were tested on trajectories different from those used for TBOD calibration, and the performance was compared against a standard EKF and a Particle Filter (PF) with systematic resampling, using 500 and 1000 particles, respectively. The MAE (\textit{Mean Absolute Error}) and RMSE (\textit{Root Mean Square Error}) of estimation error on $x/y/z$ direction and on the yaw angle are reported in \cref{res:tab:ResLab04MAE} and \cref{res:tab:ResLab04RMSE}. Roll and pitch angles are not reported due to negligible variation during the experiment, constrained by the lab setup. Results from a representative trajectory outside of the \textit{training set} $\mathcal{T}$, for the $(X, Y)$ position and Yaw angle are shown in \cref{res:fig:LabResultsLab04}. The PF trajectories have not been plotted for the sake of clarity. From \cref{res:tab:ResLab04MAE} and \cref{res:tab:ResLab04RMSE}, several trends emerge:

\begin{table}[b]
    \centering
    \begin{tabular}{|c|c|}
        \bottomrule
        \textbf{Parameter} & \textbf{Value} \\
        \toprule
         \bottomrule
        $[k_1 \ k_2 \ k_3]$ & [0.91 \ 2.17 \ -0.71]  \\
        $[k_6 \ k_5 \ k_6]$ & [0.98 \ 0.98 \ 0.98] \\
        \toprule 
    \end{tabular}    
    \caption{Hybrid observer gains obtained with the TBOD.\vspace{-10pt}}
    \label{res:tab:gains}
\end{table}

\begin{figure*}[t]
    \centering
    \includegraphics[width=\textwidth]{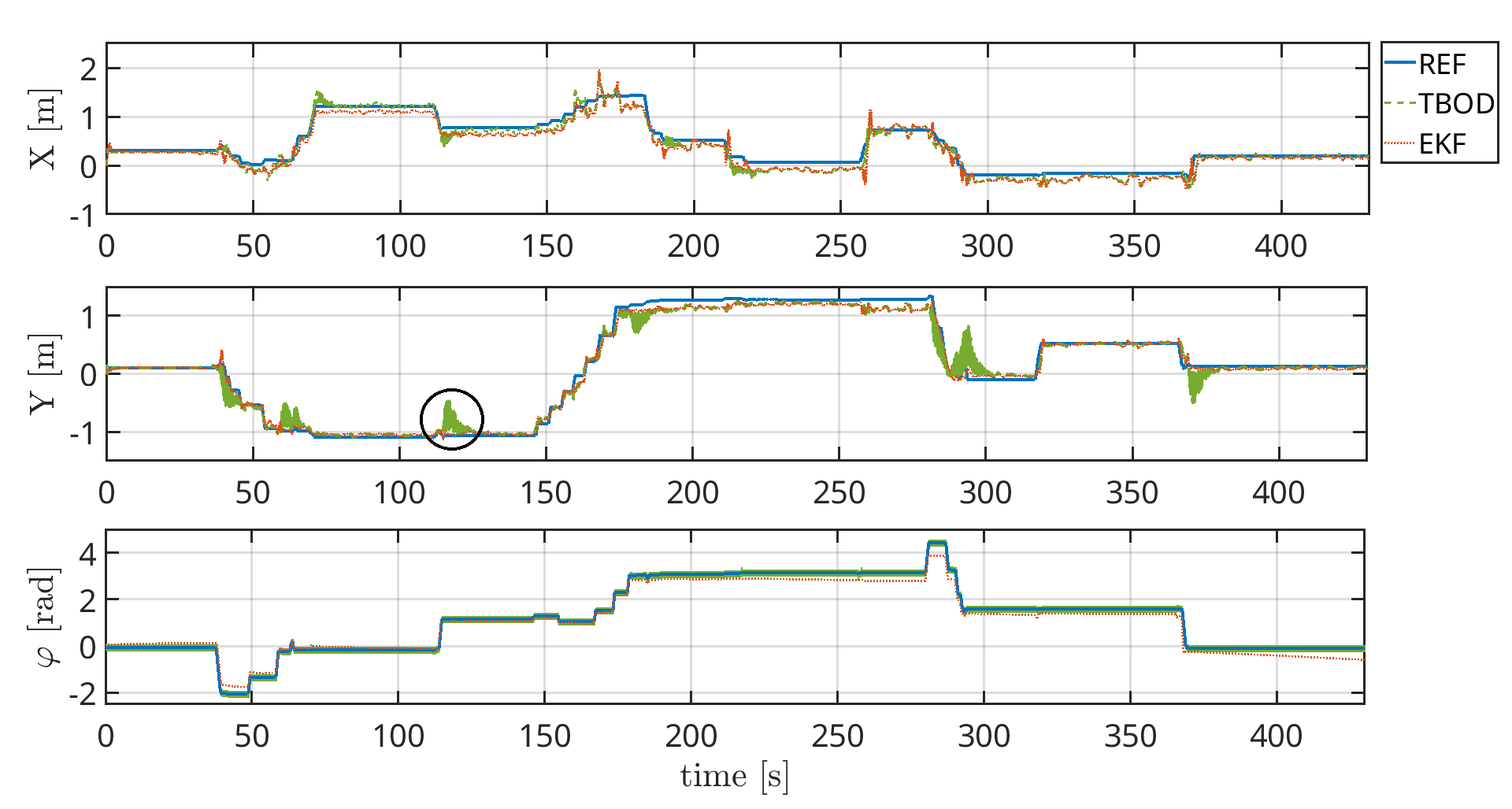}
    \vspace{-10pt}\caption{$(X, Y,\varphi)$ estimation on a test trajectory. The ground truth is in solid blue, the EKF is in dotted red, and the TBOD is in dashed green.}
    \label{res:fig:LabResultsLab04}
\end{figure*}

\begin{itemize}
\item \textbf{Position:} The TBOD and EKF show comparable performance in position estimation. Based on MAE  (\cref{res:tab:ResLab04MAE}), TBOD slightly outperforms EKF in the x-direction ($0.095 m$ vs. $0.117m$, $\simeq19\%$ improvement) and z-direction ($0.127 m$ vs. $0.107m$, though slightly worse), while EKF has an advantage in the y-direction ($0.062m$ vs. $0.081m$, $\simeq 24\%$ better). PF with 1000 particles achieves the lowest MAE in the z-axis $(0.095m)$, but shows higher variability across axes compared to TBOD. PF with 500 particles performs weakest overall, confirming its sensitivity to particle count. Similar trends hold in RMSE (\cref{res:tab:ResLab04RMSE}), where TBOD shows the lowest error in the x-axis $(0.123m)$, competitive performance in y $(0.119m)$, and slightly higher z error $(0.214m)$, still within acceptable bounds.
\item \textbf{Orientation (Yaw):} This is where TBOD yields the most significant advantage. According to MAE, TBOD achieves a yaw error of $0.25^\circ$, compared to $5.3^\circ$ for EKF and $2.12^\circ$ for PF with 1000 particles. This represents an improvement of over $95\%$ vs. EKF and $\simeq 88\%$ vs. PF-N1000. RMSE values show similar results: TBOD ($1.42^\circ$) improves upon EKF ($10.9^\circ$) by $\simeq 87\%$ and PF-N1000 ($2.66^\circ$) by $\simeq 47\%$. These figures underscore TBOD’s robustness and accuracy in orientation estimation, even compared to probabilistic methods with high computational costs.
\end{itemize}

\begin{table}[h!]
    \centering
    \begin{tabular}{|c|c|c|c|c|}         
         \hline
         \textbf{Method} & $\bm{x}$ \textbf{axis} & $\bm{y}$ \textbf{axis} & $\bm{z}$ \textbf{axis} & $\bm{\varphi}$ \textbf{angle}\\
         \toprule
         \bottomrule
         EKF & 0.117 m & 0.062 m & 0.107 m & $5.3^\circ$\\[1pt]
         PF - N500 & 0.171 m & 0.086 m & 0.168 m & $4.55^\circ$\\[1pt] 
         PF - N1000 & 0.115 m & 0.135 m & 0.095 m & $2.12^\circ$\\[1pt] 
         TBOD & 0.095 m & 0.081 m & 0.127 m & $0.25^\circ$\\
         \hline
    \end{tabular}
    \caption{Comparison of the EKF, PF, and TBOD using the MAE.}\vspace{-10pt}
    \label{res:tab:ResLab04MAE}
\end{table}

\begin{table}[h!]
    \centering
    \begin{tabular}{|c|c|c|c|c|}         
         \hline
         \textbf{Method} & $\bm{x}$ \textbf{axis} & $\bm{y}$ \textbf{axis} & $\bm{z}$ \textbf{axis} & $\bm{\varphi}$ \textbf{angle}\\
         \toprule
         \bottomrule
         EKF & 0.139 m & 0.084 m & 0.141 m & $10.9^\circ$\\[1pt]
         PF - N500 & 0.175 m & 0.086 m & 0.168 m & $4.89^\circ$\\[1pt] 
         PF - N1000 & 0.130 m & 0.137 m & 0.095 m & $2.66^\circ$\\[1pt] 
         TBOD & 0.123 m & 0.119 m & 0.214 m & $1.42^\circ$\\
         \hline
    \end{tabular}
    \caption{Comparison of the EKF, PF, and TBOD using the RMSE.}\vspace{-10pt}
    \label{res:tab:ResLab04RMSE}
\end{table}

\medskip
In terms of computational cost, TBOD and EKF exhibit similar real-time performance. In contrast, the PF is significantly more computationally expensive: even with 500 particles (PF-N500), it requires an order of magnitude more processing time than TBOD or EKF. This makes PF less suitable for real-time embedded implementations, unless additional optimization or particle reduction strategies are employed. Thus, for a lightweight real-time solution, EKF and TBOD remain the two best candidates. \cref{res:fig:LabResultsLab04} focuses on the estimation comparison between EKF and TBOD:

\begin{itemize}
    \item \textbf{Position}: coherently with the metrics from \cref{res:tab:ResLab04MAE} and \cref{res:tab:ResLab04RMSE}, the EKF and TBOD estimation are comparable. The chattering observed in the TBOD convergence (e.g., the black circle on the Y axis) illustrates both the hybrid nature of the observer, manifested through periodic position resets every $\Delta t_d$, and the structure of the bias estimation. Specifically, the bias is the only state held constant during the continuous (flow) phase and updated only during discrete jumps, based on UWB measurements. This discontinuous update scheme causes abrupt corrections, which in turn affect the integrated velocity estimate and contribute to the chattering. Introducing a continuous integration term for the bias would likely smooth the estimation, but a simpler structure was chosen during the observer design phase. As for convergence behavior, a lower $k_3$ leads to slower yet smoother convergence, whereas a more aggressive setting (higher $k_3$) improves convergence speed at the cost of increased noise. This tradeoff is precisely what the optimization procedure aims to balance.
    \item \textbf{Orientation (Yaw)}: the high level of trust in the position estimate (i.e., high $k_{4-6}$ values) results in a very accurate bearing estimate by the TBOD, with the EKF performing poorly, manifesting drifts and inaccurate values. 
\end{itemize}

\medskip
In summary, TBOD offers a strong trade-off: it achieves comparable or better position accuracy than EKF and PF while significantly improving yaw estimation and maintaining low computational overhead. These considerations hold for both the MAE and RMSE metrics, showing how TBOD is consistently more accurate, both on average and in worst-case deviations.

% CONCLUSIONS
\section{Conclusions}
\label{concl}

This work proposes a methodology for observer design called \textit{Trajectory Based Optimization Design} (TBOD). The goal is to provide a robust design procedure to design and tune lightweight and modular observers for general nonlinear systems (see \cref{meth}). An implementation example is proposed for the localization problem of a terrestrial rover in \cref{expset}, where stability guarantees for the error convergence are also provided. The observer is tested in a real scenario against a standard EKF and PF in \cref{res}, matching its accuracy on the position estimate, and improving it by an order of magnitude in the orientation. 

\medskip
In this regard, the TBOD method represents a valid approach for designing sensor-fusion solutions. Specifically, the TBOD method can design different observers for different sensor setups and noise situations. These observers can be later alternated depending on the current estimation scenario. Future work in this direction is being developed, and we hope to provide positive results soon. 

\printcredits

% \bibliographystyle{cas-model2-names}
\bibliographystyle{elsarticle-num}

\bibliography{ref}

\vfill\null
\vspace{100pt}
\bio{fig/people/federico}
Federico Oliva is a postdoctoral researcher at the Civil, Environmental, and Agricultural Robotics Lab (CEAR) at the Technion – Israel Institute of Technology, focusing on multi-agent active localization and robotic aggregate formation techniques. He holds a Master's in Automation Engineering from Alma Mater Studiorum – Università di Bologna (2017-2020), where he graduated with honors. He completed his Ph.D. at Università di Roma Tor Vergata (2020-2023) in Computer Science, Control, and Geoinformation, focusing on nonlinear observers and system identification.
\endbio

\bio{fig/people/tom}
Tom Shaked is a registered architect with a B.Arch and M.A. from Tel Aviv University, where he lectured and managed the Digital Fabrication Lab. A two-time Azrieli research fellow, he co-founded CTwiz, a real estate analytics startup, and was Head of Education at Infolio (2014-2016). He earned his Ph.D. from the Technion, focusing on robotic tools for adaptive construction. Tom has held research and teaching roles at TUM and the Technion's CEAR Lab. Now a Senior Lecturer at Ariel University, he directs ARCA Lab, researching collective robotic construction and autonomous systems in unstructured environments.
\endbio

\bio{fig/people/daniele}
Daniele Carnevale was born in Italy in 1978. He received an M.S. degree in electrical engineering and a Ph.D. in robotics applied to surgery from the university of Rome “Tor Vergata,” Rome, Italy, in 2003 and 2007, respectively.
He was a Visiting Scholar at the Center for Control Engineering and Computation, university of California, Santa Barbara 2005 and with the Electrical and Electronic Engineering Department, Imperial College, London, U.K., in 2006. He is currently an Associate Professor with the university of Rome “Tor Vergata”. His research interests are focused on nonlinear observer design, mathematical modeling of hysteresis, extremum seeking, and networked controlled systems.
\endbio

\bio{fig/people/amir}
Amir Degani is an Associate Professor at the Technion – Israel Institute of Technology and Director of the CEAR Lab. He earned his B.Sc. in Mechanical Engineering from the Technion (2002, summa cum laude) and his M.S. and Ph.D. in Robotics from Carnegie Mellon University (2006, 2010). His research focuses on mechanism design, motion planning, and nonlinear dynamic hybrid systems with civil and agricultural robotics applications. Prof. Degani holds six robotics patents and has received several awards, including IEEE Best Paper and Best Video honors. He has also been an associate editor for IEEE T-RO, ICRA, and IROS.
\endbio

\end{document}